\documentclass[letterpaper, 10 pt, conference]{formatting/ieeeconf}
\IEEEoverridecommandlockouts                              
\let\labelindent\relax
\usepackage{amsfonts}       

\usepackage{amsthm}
\usepackage{mathtools}      
\usepackage{amssymb}        
\usepackage{filecontents}
\usepackage{graphicx}       
\usepackage{marginnote}     
\usepackage{marvosym}       
\usepackage{overpic}        
\usepackage{tabularx}
\usepackage{cite}
\usepackage{color}
\usepackage[linesnumbered,algoruled,boxed,lined,noend]{algorithm2e}
\usepackage[normalem]{ulem}
\usepackage{epstopdf}
\usepackage{enumitem}
\usepackage[normalem]{ulem}
\usepackage{lipsum}
\usepackage[a-2b,mathxmp]{pdfx}[2018/12/22]

\newif\ifdraft
\draftfalse

\ifdraft
\usepackage[paperheight=11in,paperwidth=9.5in,
			left=1.25in,right=1.25in,
			top=0.75in,bottom=0.75in,
			heightrounded,marginparwidth=1.2in,
			marginparsep=0.05in]{geometry}
\usepackage{xcolor}
\usepackage{xargs} 
\usepackage[textsize=footnotesize]{todonotes}
\newcommandx{\nt}[2][1=]{\todo[linecolor=red,
			backgroundcolor=red!10,bordercolor=red,#1]{ #2}}
\newcommandx{\jy}[2][1=]{\todo[linecolor=green,
			backgroundcolor=green!10,bordercolor=green,#1]{JY: #2}}
\else
\newcommand{\nt}[1]{{}}
\newcommand{\jy}[1]{{}}
\fi

\newif\iftwocolumn
\twocolumntrue

\newtheorem{proposition}{Proposition}[section]

\theoremstyle{definition}

\theoremstyle{remark}

\SetKwProg{Fn}{Function}{}{}
\SetKwComment{Comment}{$\triangleright$\ }{}

\makeatletter
\def\subsubsection{\@startsection{subsubsection}
                                 {3}
                                 {\z@ \hspace*{1mm}}
                                 {0ex plus 0.1ex minus 0.1ex}
                                 {0ex}
                                 {\normalfont\normalsize\itshape}}
\makeatother

\newcommand{\mpp}{\textsc{MRPP}\xspace}
\newcommand{\ilp}{\textsc{ILP}\xspace}

\newcommand{\bcpr}{\textsc{BVPR}\xspace}
\newcommand{\ccpr}{\textsc{CVPR}\xspace}

\newcommand{\crp}{\textsc{VSP}\xspace}

\newcommand{\mcp}{\textsc{MCP}\xspace}
\newcommand{\csmp}{\textsc{CSMP}\xspace}
\newcommand{\makespan}{\textsc{MKPN}\xspace}
\newcommand{\anm}{\textsc{ANM}\xspace}
\newcommand{\aprt}{\textsc{APRT}\xspace}

\newif\ifarxiv
\arxivfalse
\arxivtrue

\title{Toward Efficient Physical and Algorithmic Design of Automated Garages
}
\author{Teng Guo   \qquad Jingjin Yu
\thanks{G. Teng, and J. Yu are with the Department of 
Computer Science, Rutgers, the State University of New Jersey, Piscataway, NJ, USA. 
Emails: {\tt\small \{ teng.guo, jingjin.yu\}@rutgers.edu}.
This work is supported in part by NSF award IIS-1845888 and an Amazon Research Award. 
}
}

\begin{document}
\bstctlcite{IEEEexample:BSTcontrol}

\maketitle
\thispagestyle{empty}
\pagestyle{empty}


\begin{abstract}
Parking in large metropolitan areas is often a time-consuming task with further implications toward traffic patterns that affect urban landscaping. 
Reducing the premium space needed for parking has led to the development of automated mechanical parking systems. 
Compared to regular garages having one or two rows of vehicles in each island, automated garages can have multiple rows of vehicles stacked together to support higher parking demands. 
Although this multi-row layout reduces parking space, it makes the parking and retrieval more complicated. 
In this work, we propose an automated garage design that supports near $100\%$ parking density. 
Modeling the problem of parking and retrieving multiple vehicles as a special class of multi-robot path planning problem, we propose associated algorithms for handling all common operations of the automated garage, including (1) optimal algorithm and near-optimal methods that find feasible and efficient solutions for simultaneous parking/retrieval and (2) a novel shuffling mechanism to rearrange vehicles to facilitate scheduled retrieval at rush hours.
We conduct thorough simulation studies showing the proposed methods are promising for large and high-density real-world parking applications.
\end{abstract}

\section{Introduction}\label{sec:intro}
The invention of automated parking systems (garages) helps solve parking issues in areas where space carries significant premiums, such as city centers and other heavily populated areas.
Nowadays, parking space is becoming increasingly scarce and expensive; a spot in Manhattan could easily surpass $200,000$ USD.
Developing garages supporting high-density parking that save space and are more convenient is thus highly attractive for economic/efficiency reasons.

In automated garages, human drivers only need to drop off (pick up) the vehicle in a specific I/O (Input/Output) port without taking care of the parking process.
Vehicles in such a system do not require ambient space for opening the doors, and can thus be parked much closer.
Moving a vehicle to a parking spot or a port is the key function for such systems.  
One of the solutions is to use robotic valets to move vehicles. Such systems are already commercially available, such as HKSTP \cite{HKSTP} in Hong Kong.
In such systems \cite{nayak2013robotic}, vehicles are parked such that they may block each other, requiring multiple rearrangements to retrieve a specific vehicle. 
Unfortunately, little information can be found on how well these systems function, e.g., their parking/retrieval efficiency. 
\begin{figure}[h]
    \centering
    \includegraphics[width=\linewidth]{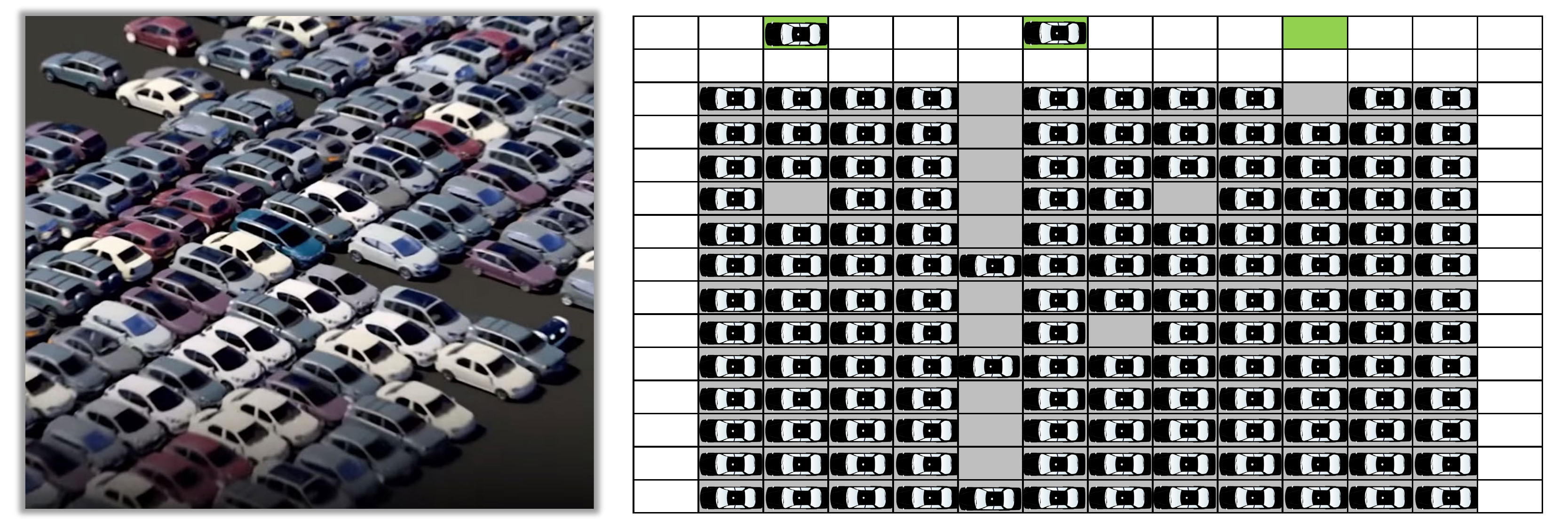}
    \caption{Left: an illustration of the density level of the envisioned automated garage system. Right: The grid-based abstraction with three I/O ports for vehicle dropoff and retrieval.}
    \label{fig:av_system}
    \vspace{-2mm}
\end{figure}%
Other solutions focus on parking for self-driving vehicles. 
In such an automated garage, vehicles are able to drive themselves  to parking slots and ports.  
This makes the system more flexible but provides limited space-saving advantages, besides requiring autonomy from the vehicles.

Recently, many efficient multi-robot path planning algorithms have been proposed, making it possible for lowering parking and retrieval cost using multiple robotic valets.
In this work, we study multi-robot based parking and retrieval problem, proposing a complete automated garage design, supporting near $100\%$ parking density, and developing associated algorithms for efficiently operating the garage. 

\textbf{Results and contributions.}
The main results and contributions of our work are as follows. In designing the automated garage, we introduce \emph{batched vehicle parking and retrieval} (\bcpr) and \emph{continuous vehicle parking and retrieval} (\ccpr) problems modeling the key operations required by such a garage, which facilitate future theoretical and algorithmic studies of automated garage systems.

On the algorithmic side, we study a system that can support parking density as high as $(m_1-2)(m_2-2)/m_1m_2$ on a $m_1\times m_2$ grid map and allow multi-vehicle parking and retrieving, which approaches $100\%$ parking density for large garages. Leveraging the regularity of the system, which is grid-like, we propose an optimal \ilp-based method and a fast suboptimal algorithm based on sequential planning.
Our suboptimal algorithm is highly scalable while maintaining a good level of solution quality, making it suitable for large-scale applications. 
We further introduce a shuffling mechanism to rearrange vehicles during off-peak hours for fast vehicle retrieval during rush hours, if the retrieval order can be anticipated. Our rearrangement algorithm performs such shuffles with total time cost of $O(m_1m_2)$ at near full garage density.

%

\textbf{Related work.} 
Researchers have proposed diverse approaches toward efficient high-density parking solutions.
Many systems for self-driving vehicles have been studied \cite{ferreira2014self,timpner2015k,nourinejad2018designing},
%
%
where vehicles are parked using a central controller and may be stacked in several rows and can block each other.
These designs increase parking capacity by up to $50\%$. However, the retrieval becomes highly complex due to blockages and is heavily affected by the maneuverability of self-driving vehicles.

With most vehicles being incapable of self-driving, robotic valet based high-density parking systems could be a more appropriate choice. 
The Puzzle Based Storage (PBS) system or  grid-based shuttle system, proposed originally by \cite{Gue2007PuzzlebasedSS}, is one of the most promising high-density storage systems.
In such a system, storage units, which can be  AGVs or shuttles, are movable in four cardinal directions. There must be at least one empty cell (escort).  To retrieve a vehicle, one must utilize the escorts to move the desired vehicles to an I/O port. This is similar to the 15-puzzle, which is known to be NP-hard to optimally solve \cite{ratner1986finding}.
Optimal algorithms for retrieving one vehicle with a single escort and multiple escorts have been proposed in \cite{Gue2007PuzzlebasedSS,optimalMultiEscort}.
However, these methods only consider retrieving one single vehicle at a time.
Besides, the average retrieval time can be much longer than conventional aisle-based solutions.
To achieve a trade-off between capacity demands and retrieval efficiency, we suggest using more escorts and I/O ports that allow retrieving and parking multiple vehicles simultaneously by utilizing recent advanced Multi-Robot Path Planning (\mpp) algorithms \cite{okoso2022high}.

\mpp has been widely studied. In the static or one-shot setting \cite{stern2019multi},  given a graph environment and a number of robots with each robot having a unique start position and a goal position,  the task is to find collision-free paths for all the robots from start to goal.
It has been proven that solving one-shot \mpp optimally in terms of minimizing either makespan or sum of costs is NP-hard \cite{surynek2009novel,yu2013structure}.
Solvers for \mpp can be categorized into \emph{optimal} and \emph{suboptimal}.  
Optimal  solvers either reduce \mpp to other well-studied problems, such as ILP\cite{yu2016optimal}, SAT\cite{surynek2010optimization} and ASP\cite{erdem2013general} or use search algorithms to search the joint space to find the optimal solution \cite{sharon2015conflict,sharon2013increasing}.
Due to the NP-hardness, optimal solvers are not suitable for solving large problems.
Bounded suboptimal solvers \cite{barer2014suboptimal} achieve better scalability while still having a strong optimality guarantee. 
However, they still scale poorly, especially in high-density environments.
There are polynomial time algorithms for solving large-scale \mpp \cite{luna2011push}, which are at the cost of solution quality. Other $O(1)$ time-optimal polynomial time algorithms \cite{GuoYuRSS22,GuoFenYu22IROS,han2018sear,yu2018constant} are mainly focusing on minimizing the makespan, which is not very suitable for continuous settings.

\textbf{Organization.}
The rest of the paper is organized as follows. Sec.~\ref{sec:problem} covers the preliminaries including garage design. 
In Sec.~\ref{sec:algo-bcpr}-Sec.~\ref{sec:algo-crp}, we provide the algorithms for operating the automated garage. We perform thorough evaluations and discussions of the garage system in Sec.~\ref{sec:evaluation} and conclude with Sec.~\ref{sec:conclusion}.

\section{Preliminaries}\label{sec:problem}
\subsection{Garage Design Specification}

In this study, the automated grid-based garage is a four-connected $m_1\times m_2$ grid $\mathcal{G}(\mathcal{V},\mathcal{E})$ (see Fig.~\ref{fig:av_system}). 
There are $n_o$ I/O ports (referred simply as \emph{ports} here on) distributed on the top border of the grid for dropping off vehicles for parking or for retrieving a specific parked vehicle.
A port can only be used for either retrieving or parking at a given time.
Vehicles must be parked at a \emph{parking spot}, a cell of the lower center $(m_1 -2)\times(m_2-2)$ subgrid. 
Once a vacant spot is parked, it becomes a movable obstacle.
%
%
$\mathcal{O}=\{o_1,...,o_{n_o}\}$ is the set of ports and $\mathcal{P}=\{p_1,...,p_{|\mathcal{P}|}\}$ is the set of parking spots.  
%

\subsection{Batched Vehicle Parking and Retrieval (\bcpr)}
\emph{Batched vehicle parking and retrieval}, or \bcpr, seeks to optimize parking and retrieval in a \emph{batch} mode. 
In a single batch, there are $n_p$ vehicles to park, and $n_r$ vehicles to retrieve, $n_l$ parked vehicles to remain.
Denote $\mathcal{C}=\mathcal{C}_p\cup\mathcal{C}_r\cup\mathcal{C}_l$ as the set of all vehicles.
At any time, the maximum capacity cannot be exceeded, i.e.,  $|\mathcal{C}|<|\mathcal{P}|$.
Time is discretized into timesteps and multiple vehicles carried by AGVs/shuttles can move simultaneously.
In each timestep, each vehicle can move left, right, up, down or wait at the current position.
Collisions among the vehicles should be avoided:
\begin{enumerate}[leftmargin=4.5mm]
    \item Meet collision. Two vehicles cannot be at the same grid point at any timestep: $\forall i,j\in\mathcal{C}, v_i(t)\neq v_j(t)$;
    \item Head-on collision. Two vehicles cannot swap locations by traversing the same edge in the opposite direction: $\forall i,j\in\mathcal{C}, (v_i(t)= v_j(t+1) \wedge v_i(t+1)=v_j(t))=\text{false}$;
    \item Perpendicular following collisions (see Fig.~\ref{fig:perp}). One vehicle cannot follow another when their moving directions are perpendicular. Denote $\hat{e}_i(t)=v_i(t+1)-v_i(t)$ as the moving direction vector of vehicle $i$ at timestep $t$, then $\forall i,j \in \mathcal{C}, (v_i(t+1)=v_j(t) \wedge \hat{e}_i(t)\perp \hat{e}_j(t))=\text{false}$.
\end{enumerate}
\begin{figure}[!htbp]
    \centering
    \includegraphics[width=.80\linewidth]{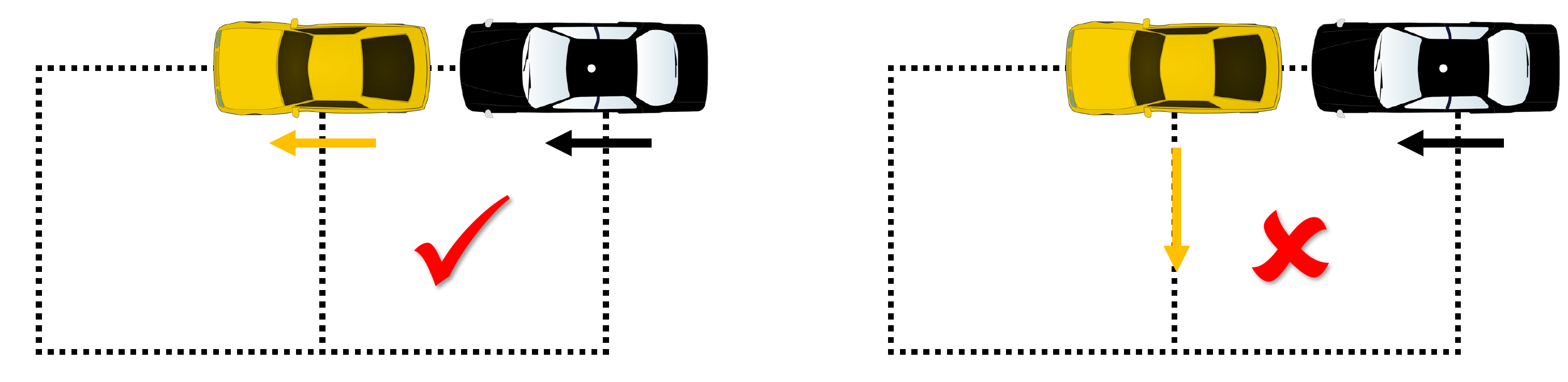}
    \caption{While parallel movement of vehicles in the same direction (left) is allowed, \emph{perpendicular following} of vehicles (right) is forbidden.}
    \label{fig:perp}
\end{figure}%

Unlike certain \mpp formulations \cite{stern2019multi}, we need to consider perpendicular following collisions, which makes the problem harder to solve.
The task is to find a collision-free path for each vehicle, moving it from its current position to a desired position.
Specifically, for a vehicle to be retrieved, its goal is a specified port. 
For other vehicles, the desired position is any one of the parking spots.
The following criteria are used to evaluate the solution quality:
\begin{enumerate}[leftmargin=4.5mm]
    \item Average parking/retrieving time (\aprt): the average time required to retrieve or park a vehicle.
    \item Makespan (\makespan): the time required to move all vehicles to their desired positions.
    \item Average number of moves per task (\anm): the sum of the distance of all vehicles divided by the total number of vehicles in $\mathcal{C}_p\cup \mathcal{C}_r$.
\end{enumerate}
In general, these objectives create  a  Pareto  front \cite{yu2013structure} and it  is  not always possible to simultaneously optimize any two of these objectives.

\subsection{Continuous Vehicle Parking and Retrieval (\ccpr)}
\ccpr is the continuous version of the vehicle parking and retrieval problem. It inherits most of \bcpr's structure, but with a few key differences.
In this formulation, we make the following assumptions. 
When a vehicle $i\in \mathcal{C}_r$ arrives at its desired port, it would be removed from the environment.
There will be new vehicles appearing at the ports that need to be parked (within capacity) and there would be new requests for retrieving vehicles.
Besides, when a port is being used for retrieving a vehicle, other users cannot park vehicles at the port until the retrieval task is finished.
Except for \makespan when the time horizon is infinite or fixed, the three criteria can still be used for evaluating the solution quality. 

\subsection{Vehicle Shuffling Problem (\crp)}
In real-world garages, there are often off-peak periods (e.g., after the morning rush hours) where the system reaches its capacity and there are few requests for retrieval.
If the retrieval order of the vehicles at a later time (e.g., afternoon rush hours) is known, then we can utilize the information to reshuffle the vehicles to facilitate the retrieval later. 
This problem is formulated as a one-shot \mpp. 
Given the start configuration $X_I$ and a goal configuration $X_G$, we need to find collision-free paths to achieve the reconfiguration. The goal configuration is determined according to the retrieval time order of the vehicles; vehicles expected to be retrieved earlier should be parked closer to the ports so that they are not blocked by other vehicles that will leave later.

\section{Solving \bcpr}\label{sec:algo-bcpr}
\subsection{Integer Linear Programming (\ilp)}
Building on network flow based ideas from 
\cite{yu2016optimal,Ma2016OptimalTA}, we reduce \bcpr to a multi-commodity max-flow problem and use integer programming to solve it.
Vehicles in $\mathcal{C}_r$ have a specific goal location (port) and must be treated as different commodities.
On the other hand, since the vehicles in $\mathcal{C}_p\cup\mathcal{C}_l$ can be parked at any one of the parking slot, they can be seen as one single commodity.
Assuming an instance can be solved in $T$ steps, we  construct a $T$-step  time-extended  network as shown in Fig.~\ref{fig:maxflow}. 

A $T$-step  time-extended network is  a  directed  network $\mathcal{N}_T=(\mathcal{V}_T,\mathcal{E}_T)$  with directed, unit-capacity edges.
The  network $\mathcal{N}_T$ contains $T+1$ copies of the original graph's vertices $\mathcal{V}$.
The copy of vertex $u\in \mathcal{V}$ at timestep $t$ is denoted as $u_t$.
At timestep $t$, an edge $(u_t,v_{t+1})$ is added to $\mathcal{E}_T$ if $(u,v)\in \mathcal{E}$ or $v=u$.
For $i\in \mathcal{C}_r$, we can give a supply of one unit of commodity type $i$ at the vertex $s_{i0}$ where $s_i$ is the start vertex of $i$.
To ensure that vehicle $i$ arrives at its  port, we add a feedback edge connecting its goal vertex $g_i$ at $T$ to its source node $s_{i0}$.
For the vehicles that need to be parked, we create an auxiliary source node $\alpha$ and an auxiliary sink node $\beta$.
For each $i\in \mathcal{C}_p \cup \mathcal{C}_l$, we add an edge of unit capacity connecting node $\alpha$ to its starting node $s_{i0}$.
As the vehicles can be parked at any one of the parking slots, for any vertex $u\in \mathcal{P}$ we add an edge of unity capacity connecting $u_T$ and $\beta$.
A supply of $n_p+n_l$ unit of commodity of the type for the vehicles in $\mathcal{C}_p\cup\mathcal{C}_l$ can be given at the node $\alpha$.

To solve the multi-commodity max-flow using \ilp, we create a set of binary variables $X=\{x_{iuvt}\}, i=0,...,n_r, (u,v)\in \mathcal{E}$ or $u=v$, $0\leq t\leq T$; a variable set to true means that the corresponding edge is used in the final solution.
The \ilp formulation is given as follows.
%
  \vspace{1mm}
\begin{equation}
      \text{Minimize \quad} \sum_{i,t,u\neq v}x_{iuvt} \label{eqn:eq_obj}
\end{equation}
\begin{equation}
    \text{subject to \quad}\forall t,v,i\quad \sum_{u}x_{iuv(t-1)}=\sum_{w}x_{ivwt}\label{eqn:flow_constraint}
\end{equation}
\begin{equation}
     \forall t,i,v\quad \sum_{v}x_{iuvt}\leq 1\label{eqn:vertex_constraint}
\end{equation}
  \begin{equation}
       \forall t,i,(u,v)\in\mathcal{E}\quad  \sum_{i}(x_{iuvt}+x_{ivut})\leq 1\label{eqn:edge_constraint}
  \end{equation}
\begin{equation}
     \forall t,i,(u,v)\perp(v,w) \sum_{i}(x_{iuvt}+x_{ivwt})\leq 1 \label{eqn:following_constraint}
\end{equation}
\begin{equation}
            \sum_{i=0}^{n_r-1} x_{ig_is_iT}+\sum_{u\in \mathcal{P}}x_{n_r u\beta T}=n_p+n_r+n_l \label{eqn:goal_constraints}
\end{equation}

\begin{equation}
          x_{iuvt}=\begin{cases}
         0&\text{if $i$ does not traverse edge $(u,v)$ at $t$ }\\
         1 &\text{if $i$  traverses edge $(u,v)$ at $t$ }\\
         \end{cases}\label{eqn:eq_binaries}
\end{equation}
\vspace{1mm}
%
\jy{When referring to equations, use eqref instead of ref, which will give you () by default.}

In Eq. \eqref{eqn:eq_obj}, we minimize the total number of moves of all vehicles within the time horizon $T$.
Eq.~\eqref{eqn:flow_constraint} specifies the flow conservation constraints at each grid point.
Eq.~\eqref{eqn:vertex_constraint} specifies the vertex constraints to avoid meet-collisions.
In Eq.~\eqref{eqn:edge_constraint}, the vehicles are not allowed to traverse the same edge in opposite directions.
Eq.~\eqref{eqn:following_constraint} specifies the constraints that forbid perpendicular following conflicts.
If the programming for $T$-step time-expanded network is feasible, then the solution is found.
Otherwise, we increase $T$ step by step until there is a feasible solution.
The smallest $T$ for which the integer programming has a solution is the minimum makespan.
As a result, the \ilp finds a makespan-optimal solution minimizing the total number of moves as  a secondary objective.
\begin{figure}[!htbp]
\vspace{2mm}
    \centering
    \includegraphics[width=.90\linewidth]{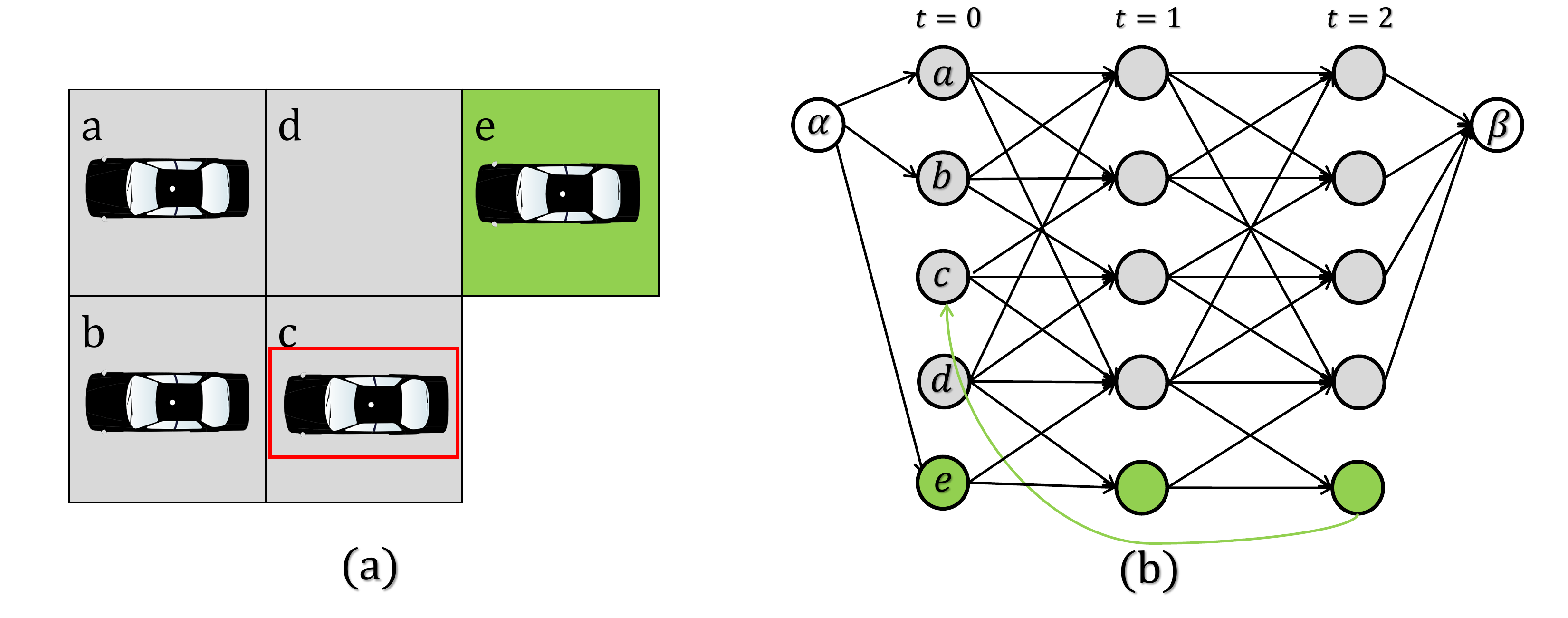}
    \vspace{-2mm}
    \caption{(a) A illustrative \bcpr instance with 4 parking spots and one port. The vehicles within the red rectangle need to be retrieved while other vehicles should be parked in the gray areas. (b) The 2-step flow network reduced from the \bcpr instance.}
    \label{fig:maxflow}
    \vspace{-2mm}
\end{figure}%
\subsection{Efficient Heuristics for High-Density Planning}
\ilp can find makespan-optimal solutions but it scales poorly. 
We seek a fast algorithm that is able to quickly solve dense \bcpr instances at a small cost of solution quality.
The algorithm we propose is built on a single-vehicle motion primitive of retrieving and parking. 
\subsubsection{Motion Primitive for Single-Vehicle Parking/Retrieving}
Regardless of the solution quality, \bcpr can be solved by sequentially planning for each vehicle in $\mathcal{C}_r\cup\mathcal{C}_p$.
Specifically, for each round we only consider completing one single task for a given vehicle $i\in\mathcal{C}_r\cup\mathcal{C}_p$, which is either moving $i$ to its port or one of the parking spots.
%
%
After vehicle $i$ arrives at its destination, the next task is solved.
%
%
The examples in Fig.~\ref{fig:single_retrieve} and Fig.~\ref{fig:single_parking}, where the maximum capacity is reached, illustrate the method's operations.

For the retrieval scenario in Fig.~\ref{fig:single_retrieve}, the vehicle marked by the red rectangle is to be retrieved.
For realizing the intuitive path indicated by the dashed lines on the left, vehicles blocking the path should be cleared out of the way, which can be easily achieved by moving those blocking vehicles one step to the left or to the right, utilizing the two empty columns. 
%
%
Such a motion primitive can always successfully retrieve a vehicle without deadlocks.

For the parking scenario in Fig.~\ref{fig:single_parking}, we need to park the vehicle in the green port.
We do so by first searching for an empty spot (escort) greedily. 
Using the mechanism of parallel moving of vehicles, the escort can first be moved to the column of the parking vehicle in one timestep and then moved to the position right below the vehicle in one timestep.
After that, the vehicle can move directly to the escort, which is its destination.
\begin{figure}[!htbp]
\vspace{1mm}
    \centering
    \includegraphics[width=\linewidth]{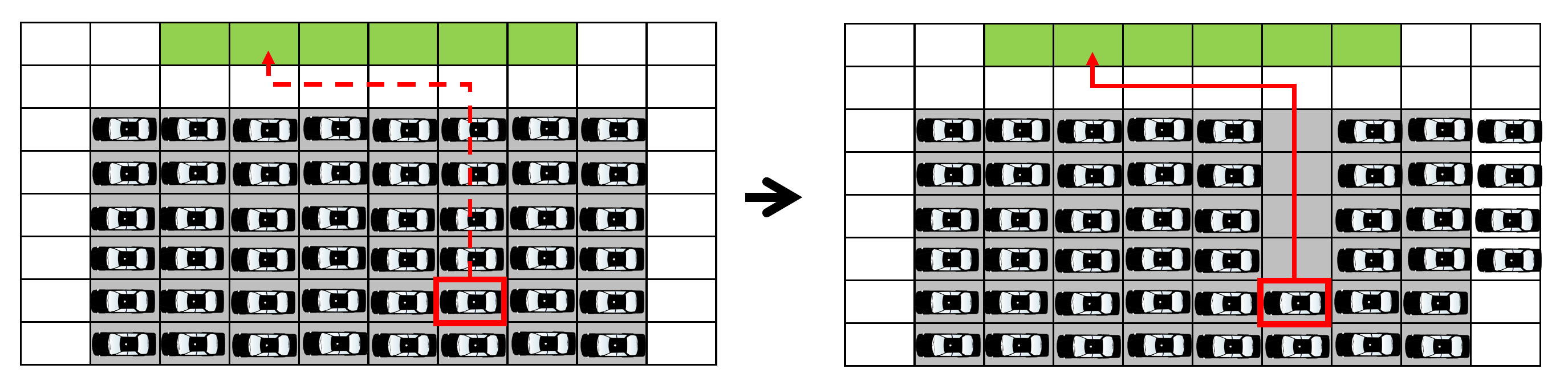}
    \caption{The motion primitive for retrieving a vehicle.}
    \label{fig:single_retrieve}
\end{figure}%

\begin{figure}[!htbp]
    \centering
    \includegraphics[width=\linewidth]{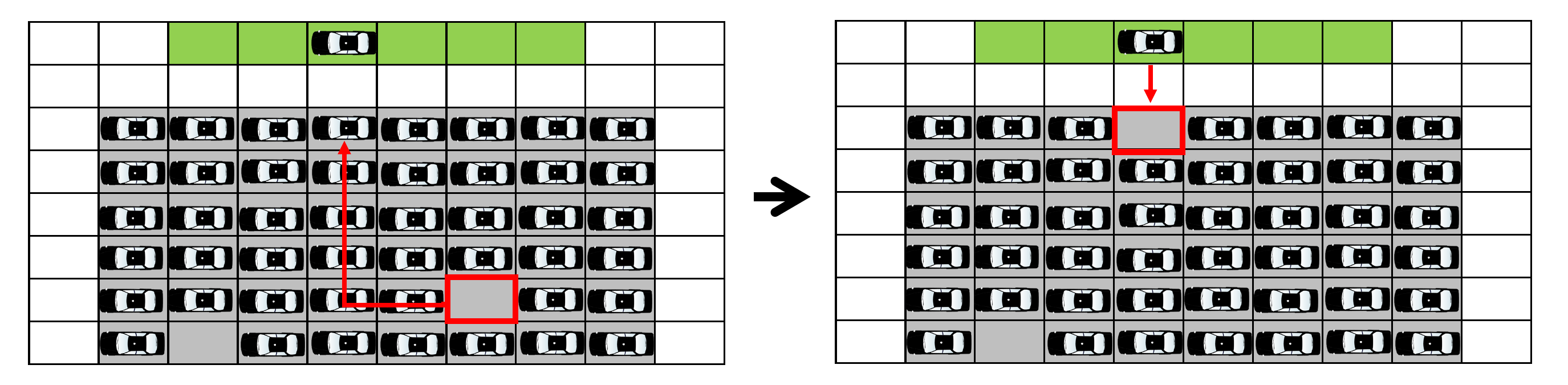}
    \caption{The motion primitive for parking a vehicle.}
    \label{fig:single_parking}
\end{figure}%

Sequential planning always returns a solution if there is one; however, when multiple parking and retrieval requests are to be executed, sequential solutions result in poor performance as measured by MKPN and ARPT.
\subsubsection{Coupling Single Motion Primitives by \mcp (\csmp)} 
Minimal  Communication  Policy (\mcp) \cite{ma2017multi} is a robust multi-robot execution policy to handle unexpected delays without stopping unaffected robots.
During  execution,  \mcp  preserves  the  order  by which robots visit each vertex as in the original plan. 
When a robot $i$ is about to perform a move action and enter a vertex $v$, \mcp checks whether robot $i$ is the next to enter that vertex by the original plan.  
If a different robot, $j$, is planned to enter $v$ next, then $i$ waits in its current vertex until $j$ leaves $v$.

We use \mcp to introduce concurrency to plans found through sequential planning.
The algorithm is described in Alg.~\ref{alg:csmp} and Alg.~\ref{alg:MCP}.
Alg.~\ref{alg:csmp} describes the framework of \csmp.
First, we find initial plans for all the vehicles in $\mathcal{C}_r\cup\mathcal{C}_p$ one by one (Line 4-6).
After obtaining the paths, we remove all the waiting states and record the order of vehicle visits for each vertex in a list of queues (Line 7).
Then we enter a loop executing the plans using \mcp until all vehicles have finished the tasks and reach their destination (Line 8-15).
In Alg.~\ref{alg:MCP}, if  $i$ is the next vehicle that enters vertex $v_i$ according to the original order, we check if there is a vehicle currently at $v_i$.
If there is not, we let $i$ enter $v_i$.
If another vehicle $j$ is currently occupying $v_j$, we examine if $j$ is moving to its next vertex $v_j$ in the next step by recursively calling the function 
$\texttt{MCPMove}$. 
If $j$ is moving to $v_j$ in the next step and the moving directions of $i,j$ are not perpendicular, we let vehicle $i$ enter vertex $v_i$.
Otherwise, $i$ should  wait at $u_i$. 
The algorithm is deadlock-free by construction; we omit the relatively straightforward proof due to the page limit.
\begin{proposition}
\csmp is dead-lock free and always finds a feasible solution in finite time  if there is one.
\end{proposition}
%
%
\begin{algorithm}[!htbp]
\begin{small}
\DontPrintSemicolon
\SetKwProg{Fn}{Function}{:}{}
\SetKwFunction{Fcsmp}{\csmp}
\SetKwFunction{FMCP}{MCPMove}
\SetKw{Continue}{continue}

 \caption{CSMP \label{alg:csmp}}
 
\Fn{\Fcsmp()}{
    \textbf{foreach} $v\in \mathcal{V}$, $VOrder[v]\leftarrow Queue()$\;
    $InitialPlans\leftarrow \{\}$\;
    \For{$i\in \mathcal{C}_p\cup\mathcal{C}_r$}{
    \texttt{SingleMP($i$,$InitialPlans$)}\;
    }
    \texttt{Preprocess($InitialPlans,VOrder$)}\;
    \While{True}{
    \For{$i\in \mathcal{C}$}{
        $mcpMoved\leftarrow Dict()$\;
        \FMCP($i$)\;
    }
    \If{\texttt{AllReachedGoal()}=true}{
    break\;
    }
    }
}
\end{small}
\end{algorithm}

\begin{algorithm}[!htbp]
\begin{small}
\DontPrintSemicolon
\SetKwProg{Fn}{Function}{:}{}
\SetKwFunction{Fcsmp}{\csmp}
\SetKwFunction{FMCP}{MCPMove}
\SetKw{Continue}{continue}

 \caption{MCPMove \label{alg:MCP}}
 
\Fn{\FMCP($i$)}{
\If{$i$ in $mcpMoved$}{
\Return $mcpMoved[i]$\;
}
$u_i\leftarrow$ current position of $i$\;
$v_i\leftarrow$ next position of $i$\;
\If{$i=VOrder[v_i].front()$}{
    $j\leftarrow$ the vehicle currently at $v_i$\;
    \If{$j$=None or (\FMCP($j$)=true and $(u_i,v_i)\not \perp (u_j,v_j)$)}{
        move $i$ to $v_i$\;
        $VOrder[v_i].popfront()$\;
        $mcpMoved[i]$=true\;
        \Return true\;
    }
}
let $i$ wait at $u_i$\;
$mcpMoved[i]$=false\;
\Return false\;
}
\end{small}
\end{algorithm}
\subsubsection{Prioritization}
Sequential planning can always find a solution regardless of the planning order.
However, priorities will affect the solution quality. 
Instead of planning by a random priority order (Alg. \ref{alg:csmp} Line 4-6), when possible, we can first plan for parking since single-vehicle parking only takes two steps.
After all the vehicles in $\mathcal{C}_p$ have been parked, we apply \texttt{SingleMP} to retrieve vehicles.
Among vehicles in $\mathcal{C}_r$, we first apply \texttt{SingleMP} for those vehicles that are closer to their port so that they can reach their targets earlier and will not block the vehicles at the lower row.
\subsection{Complexity Analysis}
In this section, we analyze the time complexity and solution makespan upper bound of the \csmp.
In \texttt{SingleMP}, in order to park/retrieve one vehicle, we assume $n_{b}$ vehicles  may cause blockages and need to be moved out of the way. Clearly $n_b<n$ where $n=n_p+n_r+n_l$.
Therefore, the complexity of computing the paths using \texttt{SingleMP} for all vehicles in $\mathcal{C}_r\cup\mathcal{C}_p$ is bounded by $(n_p+n_r)n$.
The path length of each single-vehicle path computed by \texttt{SingleMP} is no more than $m_1+m_2$.
The makespan of the paths obtained by concatenating all the single-vehicle paths is bounded by $n_r(m_1+m_2)+2n_p$.
This means that \mcp will take no more than $n_r(m_1+m_2)+2n_p$ iterations.
Therefore, the makespan of the solution is upper bounded by $n_r(m_1+m_2)+2n_p$.
In each loop of \mcp, we essentially run DFS on a graph that has $n=n_p+n_r+n_l$ nodes and traverse all the nodes, for which the time complexity is $O(n)$,
Therefore the time complexity of \csmp  is $O(n(n_rm_1+n_rm_2+2n_p))$.
In summary, the time complexity of \csmp is bounded by $O(n(n_rm_1+n_rm_2+2n_p))$, while the makespan is upper bounded by $n_r(m_1+m_2)+2n_p$.

\subsection{Extending \csmp to \ccpr}\label{sec:algo-ccpr}
CSMP can be readily adapted to solve \ccpr.
%
Similar to the \bcpr version, we call \texttt{MCPMove} for each vehicle at each timestep.
When a new request comes at some timestep, we compute the paths for the associated vehicles using the \texttt{SingleMP} and update the information of vertex visit order. 
That is, when we apply \texttt{SingleMP} on a vehicle $i \in \mathcal{C}_p\cup\mathcal{C}_r$, if vehicle $j$ will visit vertex $u$ at timestep $t'$, then we push $i$ to the queue $VOrder[u]$, where the queue is always sorted by the entering time of $u$.
In this way, \mcp will execute the plans while maintaining the visiting order. The previously planned vehicles will not be affected by the new requests and keep executing their original plan.
The main drawback of this method is that it usually has worse solution quality than replanning since the visiting order is fixed.

\section{Solving \crp via Rubik Tables}\label{sec:algo-crp}
%
\crp is essentially solving a static/one-shot \mpp.
On an $m_1\times m_2$ grid, it can be solved by applying the Rubik Table algorithm \cite{szegedy2020rearrangement}, using no more than $2(m_2-2)$ column shuffles and $(m_1-2)$ row shuffles. 
As an example shown in Fig.~\ref{fig:hw_heuristic}, we may use two nearby columns to shuffle the vehicles in a given column fairly efficiently, requiring only $O(m_1)$ steps \cite{GuoYuRSS22}.
%
%
%
%
The same applies to row shuffles.
Depending on the number of parked vehicles, one or more multiple row/column shuffles may be carried out simultaneously. We have (straightforward proofs are omitted due to limited space)
\begin{proposition}
\label{p:arxiv_prop}
\crp may be solved using $O(m_1m_2)$ makespan at full garage capacity and $O(m_1 + m_2)$ makespan when the garage has $\Theta(m_1m_2)$ empty spots and $\Omega(m_1m_2)$ parked vehicles. In contrast, with $\Omega(m_1m_2)$ parked vehicles, the required makespan for solving \crp is $\Omega(m_1 + m_2)$.
\end{proposition}

\begin{figure}[!htbp]
    \centering
    \includegraphics[width=\linewidth]{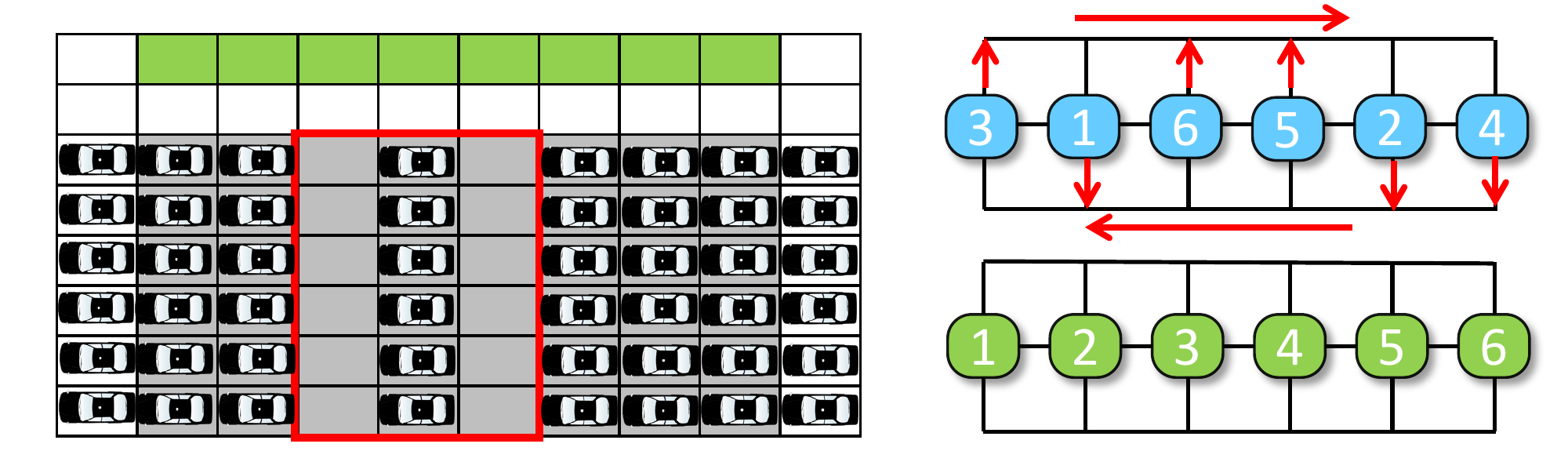}
    \caption{Illustration of the mechanism for ``shuffling'' a single row/column.}
    \label{fig:hw_heuristic}
\end{figure}%

\nt{I recall we had some good discussions and positive results on the 2/3 density setting. These could enhance the paper but we probably don't have space to accommodate them. I am keeping a note here for future reference.}


\section{Evaluation}\label{sec:evaluation}
In  this  section,  we  evaluate the  proposed algorithms.
All experiments are performed on an Intel\textsuperscript{\textregistered} Core\textsuperscript{TM} i7-9700 CPU at 3.0GHz. Each data point is an average over 20 runs on randomly generated instances unless otherwise stated.
\ilp is implemented in C++ and other algorithms are implemented in CPython. 
A video of the simulation can be found at \url{https://youtu.be/XPpOB5f7CzA}.

\subsection{Algorithmic Performance on \bcpr}
\textbf{Varying grid sizes.}
In the first experiment, we evaluate the proposed algorithms on $m\times m$ grids with varying grid side length, under the densest scenarios: there are $(m-2)^2$ vehicles in the system and all the ports are used for either parking or retrieving ($n_p+n_r=n_o$). 
The result can be found in Fig.~\ref{fig:size_exp}.
CONCAT is the method that simply concatenates the sing-vehicle paths. 
In rCSMP, we apply \texttt{SingleMP} on vehicles with random priority order, while in p\csmp, we apply the prioritization strategy.
\begin{figure}[!htbp]
\vspace{2mm}
    \centering
    \includegraphics[width=1.0\linewidth]{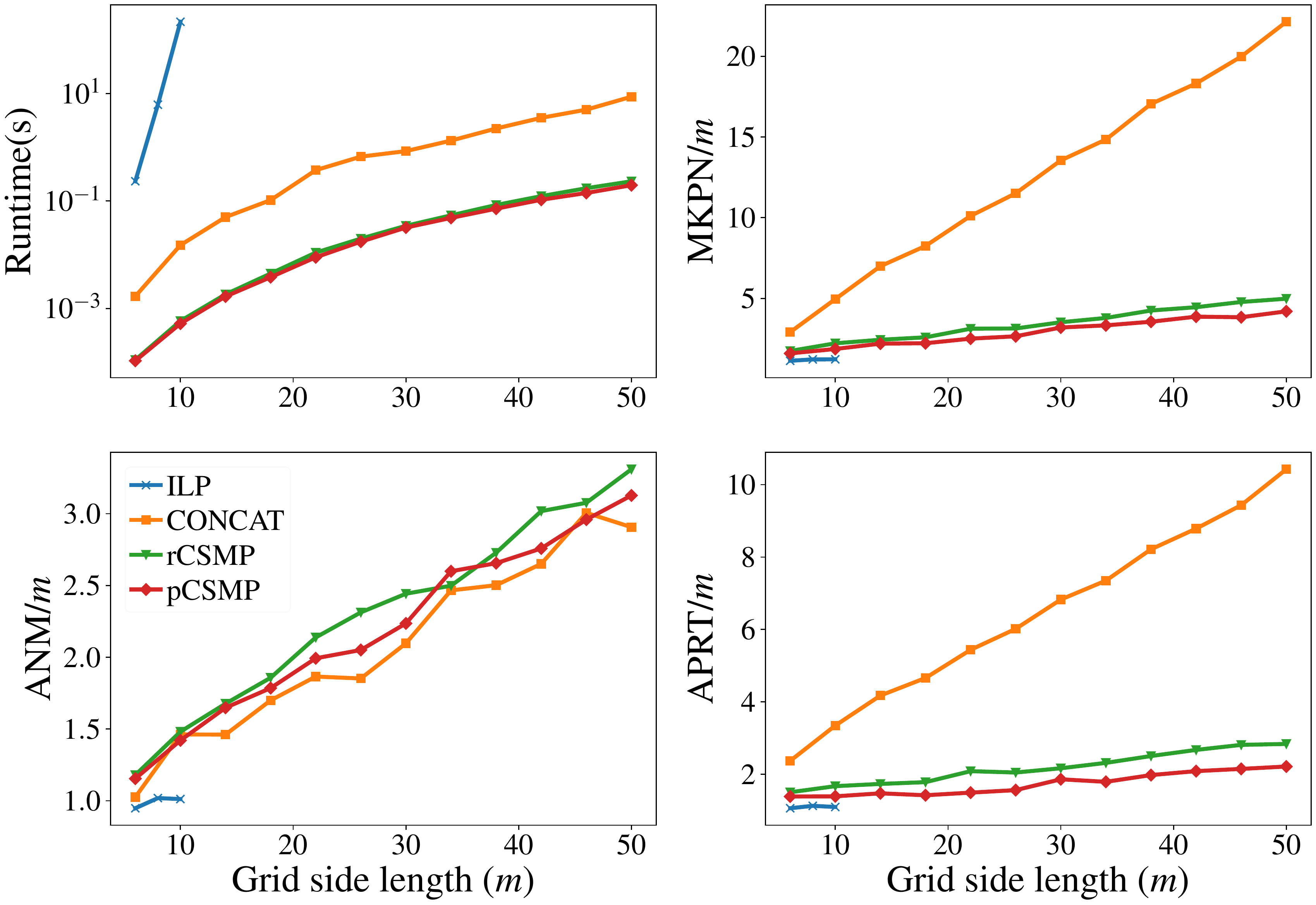}
    \caption{Runtime, MKPN, APRT, AVN data of the proposed methods on $m\times m$ grids under the densest scenarios.}
    \label{fig:size_exp}
\end{figure}%

Among the methods, \ilp has the best solution quality in terms of MKPN, ANM and APRT, which is expected since its optimality is guaranteed. 
However, \ilp has the poorest scalability, hitting a limit with $m\leq 10$ and $n\leq64$.
CONCAT, r\csmp and p\csmp are much more scalable, capable of solving instances on $50\times50$ grids with 2304 vehicles in a few seconds.
%
%
Since CONCAT just concatenates sing-vehicle paths, this results in very long paths compared to r\csmp and p\csmp; we observe that the \mcp procedure greatly improves the concurrency, leading to much better solution quality.
MKPN and APRT of paths obtained by CONCAT can be $10m$-$20m$ while the MKPN and APRT of the paths obtained by r\csmp and p\csmp are $2m$-$4m$.
MKPN and APRT of p\csmp with prioritization strategy is about $20\%$ lower than these of r\csmp. 

\textbf{Impact of vehicle density.}
In the second experiment, we examine the behavior of algorithms as vehicle density changes, fixing grid size at $20\times 20$. We still let $n_p+n_r=n_o$.
The result is shown in Fig.~\ref{fig:density_exp}.
As in the previous case, \ilp can only solve instances with density below $20\%$ in a reasonable time, while the other three algorithms can all tackle the densest scenarios.
For all algorithms, vehicle density in $\mathcal{C}_l$ has limited impact on MKPN and APRT, where r\csmp and p\csmp have much better quality than CONCAT. 
In low-density scenarios, fewer vehicles need to move which may cause blockages for retrieving/parking a vehicle.
As a result, ANM increases as vehicle density increases.

\begin{figure}[!htbp]
    \centering
    \includegraphics[width=1.0\linewidth]{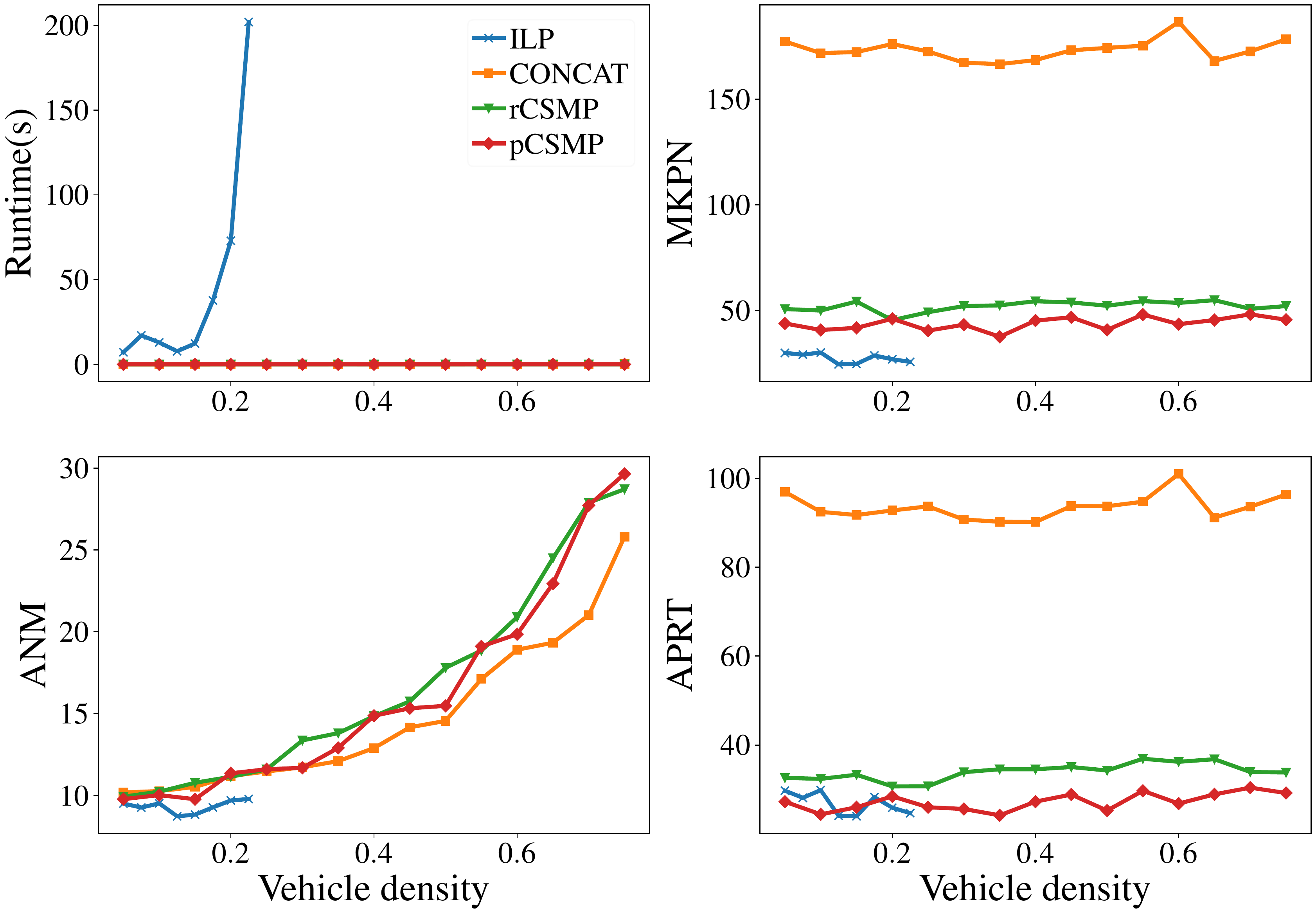}
    \vspace{-4mm}
    \caption{Runtime, \makespan, \aprt, \anm data of the proposed methods on $20\times 20$ grids with varying vehicle density.}
    \label{fig:density_exp}
\end{figure}%
    \vspace{-1.5mm}
\subsection{Algorithmic Performance on \ccpr}
    \vspace{-1mm}
\textbf{Random retrieving and parking.}
We test the continuous \csmp on a $12\times 12$ grid with $10$ ports. 
In each time step, if a port is available,  there would be a new vehicle that need to be parked appearing at this port with probability $p_p$ if it does not exceed the capacity. And with probability $p_r$ this port will be used to retrieve a random parked vehicle if there is one.
We simulate the following three scenarios:

(i). Morning rush hours. Initially, no vehicles are parked. There are many more requests for parking than retrieving: $p_p=0.6,p_r=0.01$.

(ii). Workday hours. Initially, the garage is full. Request for parking and retrieval are equal: $p_p=p_r=0.05$.

(iii). Evening rush hours. Initially, the garage is full. Retrieval requests dominate parking: $p_p=0.01, p_r=0.6$.

The maximum number of timesteps is set to 500. We evaluate the average retrieval time, average parking time, and total number of moves under these scenarios. The result is shown in Fig.~\ref{fig:online_scenarios}.
Online CSMP achieves the best performance in the morning due to fewer retrievals.
On the other hand, the average retrieval time in all three scenarios is less than $2m$ and parking time is less than $m$. This shows that the algorithm is able to plan paths with good solution quality even in the densest scenarios and rush hours.

\begin{figure}[!htbp]

    \centering
    \includegraphics[width=1.0\linewidth]{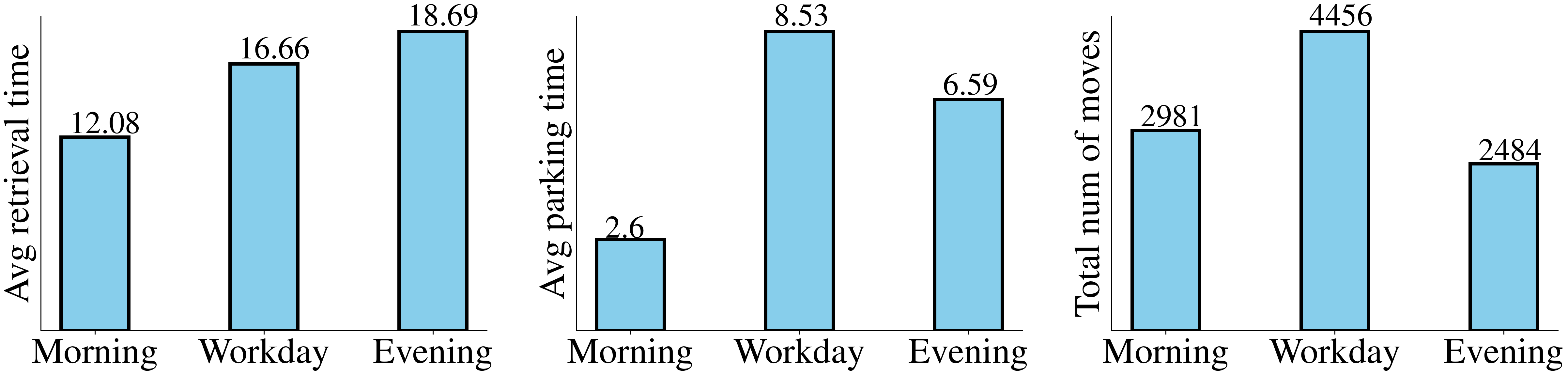}
    \caption{\ccpr performance statistics under three garage traffic patterns.}
    \label{fig:online_scenarios}
    \vspace{-2mm}
\end{figure}%

\textbf{Benefits of shuffling.}
In this experiment, we examine the effect of the shuffling (for solving \crp).
We assume that each vehicle is assigned a retrieval priority order, as to be expected in the evening rush hours when some people go home earlier than others. 
We perform the \emph{column} shuffle operations on the vehicles to facilitate the retrieval.
The makespan, average number of moves, and computation time of the column shuffle operations on $m\times m$ grids with different grid sizes under the densest settings are shown in Fig.~\ref{fig:shuffles}(a)-(c).
While the paths of shuffling can be computed in less than 1 second, the makespan of completing the shuffles scales linearly with respect to $m^2$ and the average number of moves scales linearly with respect to $m$. 

\begin{figure}[!htbp]
\vspace{-1mm}
    \centering
    \includegraphics[width=1.0\linewidth]{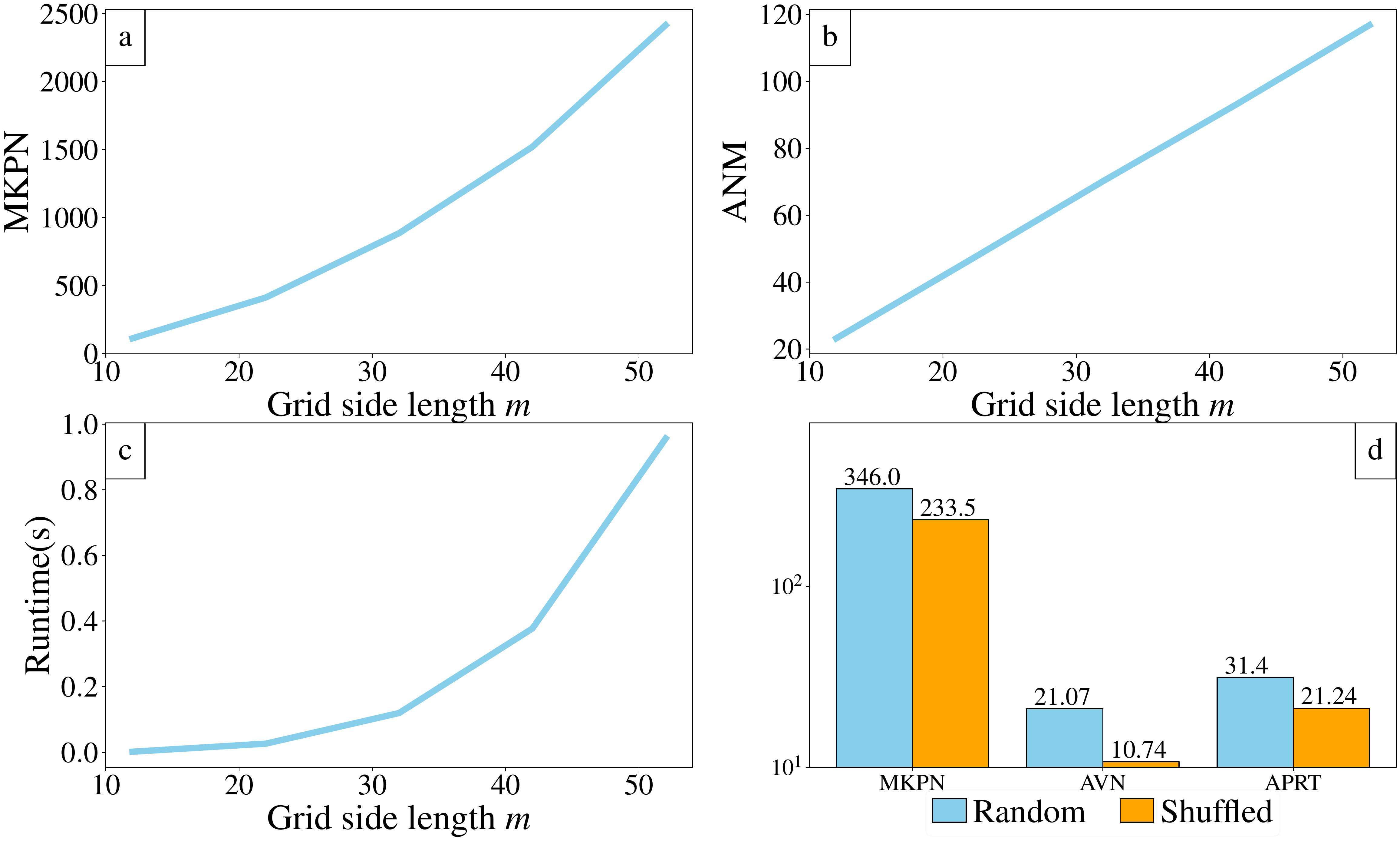}
    \caption{Statistics of performing the column shuffle operations on $m\times m$ grids with varying grid size.}
    \label{fig:shuffles}
\end{figure}%
After shuffling, continuous \csmp with $p_r=1,p_p=0$ is applied to retrieve all vehicles and compared to the case where no shuffling is performed. 
The outcome makespan, average number of moves per vehicle, and average retrieval time per vehicle are shown in Fig.~\ref{fig:shuffles}(d).
Compared to unshuffled configuration, \csmp is able to retrieve all the vehicles with $30\%$-$50\%$ less number of moves and  retrieval time 
(note that logarithmic scale is used to fit all data), showing that rearranging vehicles in anticipation of rush hour retrieval provides significant benefits.  


\vspace{-1mm}
\section{Conclusion and Discussions}\label{sec:conclusion}
\vspace{-1mm}
In this work, we present the complete physical and algorithmic design of an automated garage system, aiming at allowing the dense parking of vehicles in metropolitan areas at high  speeds/efficiency. 
We model the retrieving and parking problem as a multi-robot path planning problem, allowing our system to support nearly $100\%$ vehicle density.
The proposed \ilp algorithm can provide makespan-optimal solutions, while \csmp algorithms are highly scalable with good solution quality. Also clearly shown is that it can be quite beneficial to perform vehicle rearrangement during non-rush hours for later ordered retrieval operations, which is a unique high-utility feature of our automated garage design. 
%

For future work, we intend to further improve \csmp's scalability and flexibility, possibly leveraging the latest advances in  \mpp/MAPF research. 
%
%
We also plan to extend the automated garage design from 2D to 3D, supporting multiple levels of parking. Finally, we would like to build a small-scale test
beds realizing the physical and algorithmic designs.

\ifarxiv
\section{Appendix}
\begin{proof}[Proof of Proposition \ref{p:arxiv_prop}]
For the full density, Rubik Table algorithm can be applied here.
To shuffle a column, one can apply the high-way motion primitive by utilizing  two empty columns or the line-merge-sort motion primitive by utilizing one empty column \cite{GuoYuRSS22}.
Each column shuffle requires $O(m_1)$ steps and there are $m_2$ such columns, and all the columns can be shuffled \emph{sequentially} using the one/two empty columns.
The same can be applied to shuffle a row.
Thus, any reconfiguration can be done using $O(m_1m_2)$ makespan. 

Next, we prove that when there are $\Theta(m_1m_2)$ escorts, it requires $O(m_1+m_2)$  makespan for the reconfiguration, using the Rubik Table Algorithm.
The key point is to find a  method to shuffle all rows/columns in $O(m_1+m_2)$.

Assume that we have $\lambda m_1m_2 (0<\lambda<1)$ vehicles and $(1-\lambda)m_1m_2$ escorts.
To do that, first, we first apply unlabeled \mpp to convert the start and goal configurations to \emph{block configurations}  (Fig.~\ref{fig:island} (d) ) which are composed of several  blocks, and each block has at least one empty column or row.
For simplicity, we assume that each black cell of the  block is filled with a vehicle.
If not, we fill in with virtual vehicles.
It is not hard to see that shuffling all the columns in each block could be done in $O(Wm_2)$ steps, 
where $W=\frac{1}{1-\lambda}=O(1)$ is the width of the block and shuffles in different blocks can be performed \emph{in parallel}.
Row shuffles can be done similarly in $O(m_1)$.

Then we only need to prove that  the unlabeled conversion can be completed in $O(m_1+m_2)$. 
The conversion can be done in the following way shown in Fig.~\ref{fig:island}.
First, starting from the initial configuration, we move all the vehicles rightwards as much as possible. (Fig.~\ref{fig:island}(a) to Fig.~\ref{fig:island}(b)).
In the second step, vehicles are moved upwards  (Fig.~\ref{fig:island} (b) to Fig.~\ref{fig:island}(c)).
To get configuration (d) from configuration (c), the extra vehicles in the rows that have more vehicles can move downwards along the column to the rows that have fewer vehicles. 
After each row has the same number of vehicles, all the vehicles can move along the rows leftwards to form the configuration (d).
\begin{figure}[!htbp]
    \centering
    \includegraphics[width=\linewidth]{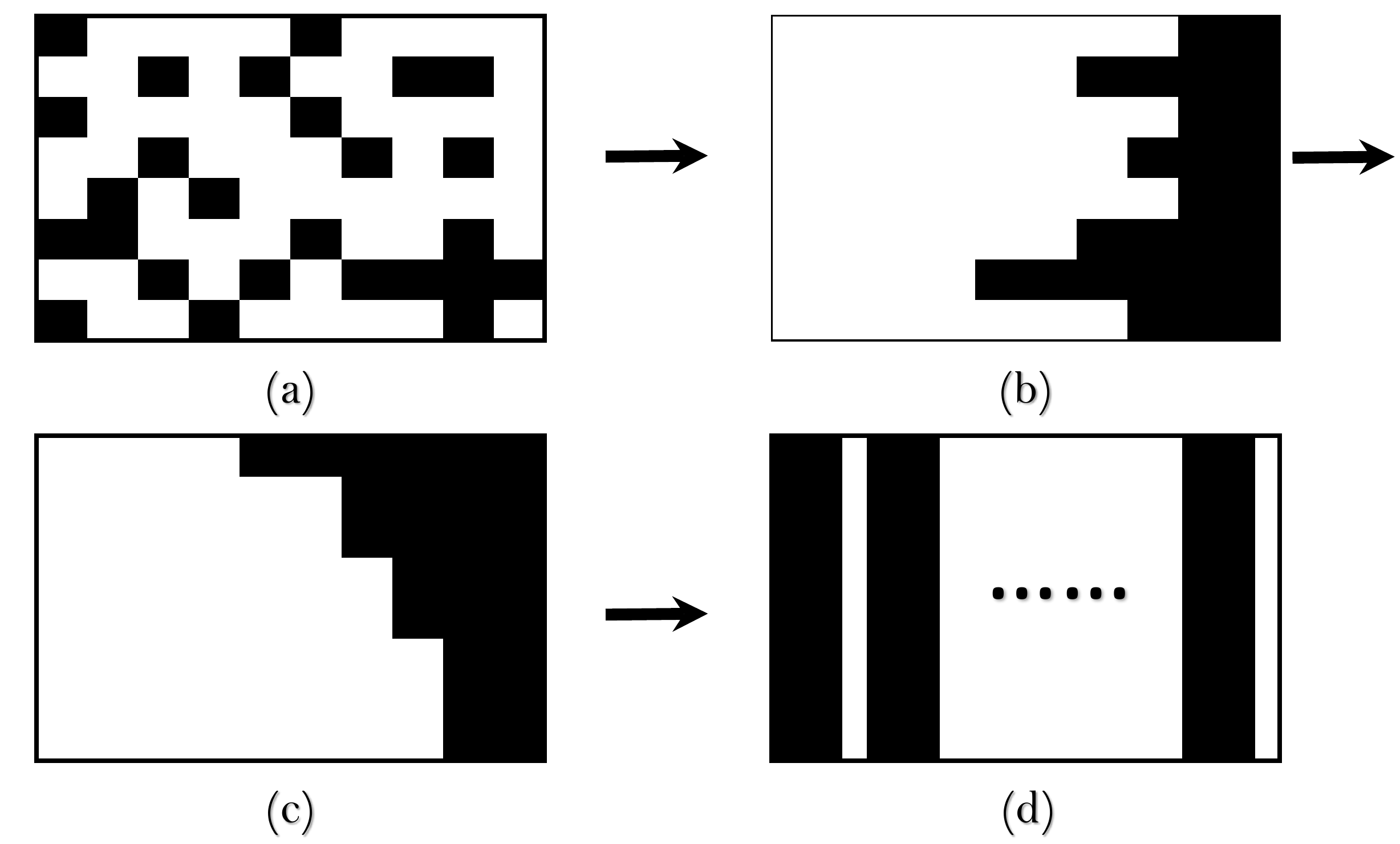}
    \caption{Steps taking an arbitrary configuration to a ``block'' configuration in $O(m_1 + m_2)$ steps, on which Rubik Table results can be applied. (a) Initial configuration. (b) Configuration obtained  by moving all vehicles rightwards from (a). (c) Configuration obtained by moving all vehicles upwards from (b). (d) The desired regular ``block'' configuration.}
    \label{fig:island}
\end{figure}%
In the whole process, all the vehicles always move in one direction, and there would be no meet, head-on, or perpendicular following conflicts happening.
Therefore, the unlabeled conversion can be completed in $O(m_1+m_2)$.
Combining all together, we conclude  that when there are $\Theta(m_1m_2)$ vehicles, the reconfiguration of VSP takes $O(m_2+m_1)$.

The lower bound can be obtained by computing the maximum  Manhattan distance between the starts and goals. 
For random starts and goals with $\Omega(m_1m_2)$ vehicles, the lower bound is $m_1+m_2-o(m_1+m_2)$ \cite{GuoYuRSS22}.
Thus, the required number of steps for the reconfiguration is $\Omega(m_1+m_2)$ when there are $\Omega(m_1m_2)$ vehicles.
%

%
%
%
%

\end{proof}
\fi

\bibliographystyle{formatting/IEEEtran}
\bibliography{bib/all}

\end{document}